\documentclass[10pt,twocolumn,letterpaper]{article}

\usepackage{cvpr}
\usepackage{times}
\usepackage{epsfig}
\usepackage{graphicx}
\usepackage{amsmath}
\usepackage{amssymb}
\usepackage{amsthm}
\usepackage{times}
\usepackage{epsfig}
\usepackage{helvet}  %Required
\usepackage{courier}  %Required
\usepackage{float}
\usepackage{array,multirow}
\usepackage{algorithm,algorithmic}
\usepackage{setspace}
\usepackage{mathtools}
\usepackage{tabularx}
\usepackage[justification=centering]{subfig}

\newtheorem{prop}{Proposition}

\usepackage[pagebackref=true,breaklinks=true,letterpaper=true,colorlinks,bookmarks=false]{hyperref}

\cvprfinalcopy % *** Uncomment this line for the final submission

\def\cvprPaperID{3967} % *** Enter the CVPR Paper ID here

\begin{document}

%%%%%%%%% TITLE
\title{Variational Bayesian Dropout with a Hierarchical Prior}

\author{Yuhang Liu$^{1,2}$, Wenyong Dong$^{1}$, Lei Zhang$^{3,2}$, Dong Gong$^{2}$, and Qinfeng Shi$^{2}$\\
$^{1}$School of Computer Science, Wuhan University, Hubei, China\\
$^{2}$The University of Adelaide, Australia~~~
$^{3}$Inception Institute of Artificial Intelligence, UAE\\
{\color{red}https://sites.google.com/view/yuhangliu/pro}
}

\maketitle
%\thispagestyle{empty}

%%%%%%%%% ABSTRACT
\begin{abstract}
Variational dropout (VD) is a generalization of Gaussian dropout, which aims at inferring the posterior of network weights based on a log-uniform prior on them to learn these weights as well as dropout rate simultaneously. The log-uniform prior not only interprets the regularization capacity of Gaussian dropout in network training, but also underpins the inference of such posterior. However, the log-uniform prior is an improper prior (i.e., its integral is infinite), which causes the inference of posterior to be ill-posed, thus restricting the regularization performance of VD. To address this problem, we present a new generalization of Gaussian dropout, termed variational Bayesian dropout (VBD), which turns to exploit a hierarchical prior on the network weights and infer a new joint posterior. Specifically, we implement the hierarchical prior as a zero-mean Gaussian distribution with variance sampled from a uniform hyper-prior. Then, we incorporate such a prior into inferring the joint posterior over network weights and the variance in the hierarchical prior, with which both the network training and dropout rate estimation can be cast into a joint optimization problem. More importantly, the hierarchical prior is a proper prior which enables the inference of posterior to be well-posed. In addition, we further show that the proposed VBD can be seamlessly applied to network compression. Experiments on classification and network compression demonstrate the superior performance of the proposed VBD in regularizing network training.
\end{abstract}

%%%%%%%%% BODY TEXT
\section{Introduction}
\label{sec:int}
Deep neural networks have gained great success in various artificial intelligence research areas, e.g., computer vision \cite{he2016deep,gong2017motion}, natural language processing \cite{collobert2008unified}, etc. Nevertheless, due to the limited samples with annotation in practice, training deep neural networks with extensive parameters often suffers from over-fitting problem~\cite{zhang2016understanding}.
{\textit{Dropout}} proves to be a practical technique to alleviate this problem, which stochastically regularizes network weights by randomly enforcing multiplicative noise on input features during training \cite{hinton2012improving}. Over the past several years, various dropout methods have been proposed
%put forward
~\cite{hinton2012improving,srivastava2014dropout,wang2013fast}.  Among them, {\textit{Gaussian dropout}}~\cite{wang2013fast} provides a general framework, which introduces the distribution of network weights into model training and thus can well approximate the conventional dropout with different types of noise, such as binary noise~\cite{hinton2012improving} or Gaussian noise~\cite{srivastava2014dropout}. While these methods have shown promising regularization performance in various deep network architectures~\cite{krizhevsky2012imagenet,simonyan2014very,huang2017densely,kendall2015bayesian,goodfellow2016deep}, the reason behind such success is not clear, and their performance heavily depends on a predefined dropout rate, for which traditional grid-search based methods is a prohibitive operation for large network models.

\begin{figure}
\center
\includegraphics[width=3.3in,height=0.75in]{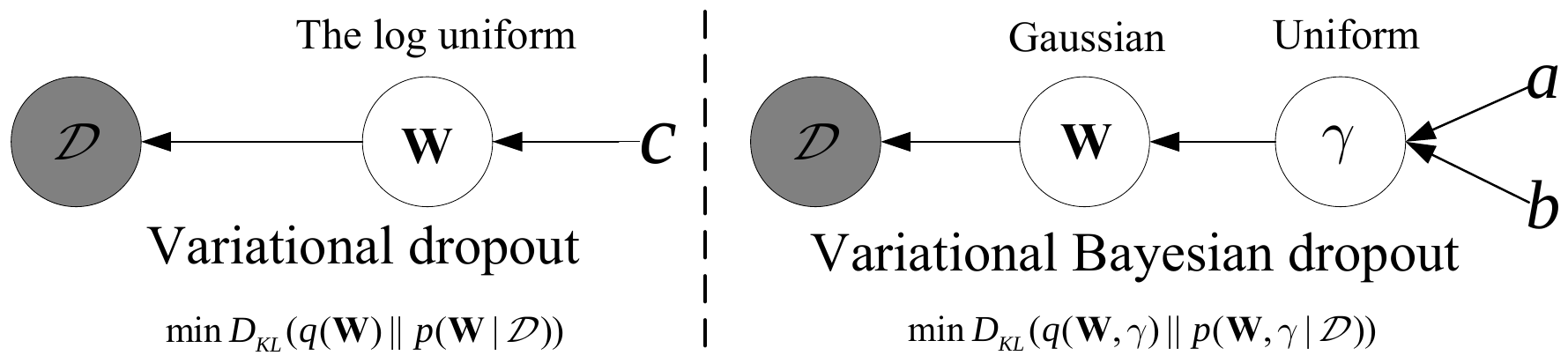}
\caption{Variational dropout v.s. the proposed variational Bayesian dropout.}
  \label{VBD}
\vspace{-0.3cm}
\end{figure}

{\textit{Variational dropout}} (VD) \cite{kingma2015variational} is a generalization of Gaussian dropout, which focuses on inferring the posterior of network weights based on a log-uniform prior.
% on them.
By doing this, VD can address these two aspects of the problems mentioned above. 1) {\textit{Bayesian Interpretation}}. It has been proved~\cite{kingma2015variational} that VD can be consistent with Gaussian dropout for a fixed dropout rate by enforcing the log-uniform prior on network weights. This implies that incorporating Gaussian dropout into network training amounts to variational inference on the network weights, where these weights are regularized by the Kullback-Leibler (KL) divergence between the variational posterior, i.e., the distribution of network weights introduced by Gaussian dropout, and the log-uniform prior. In other words, the log-uniform prior endows Gaussian dropout with the regularization capacity.
2) {\textit{Adaptive dropout rate}}. Based on the log-uniform prior, VD~\cite{kingma2015variational} can simultaneously learn network weights as well as dropout rate via inferring the posterior on these weights.
To sum up, the log-uniform prior as the footstone of VD underpins these two advantages above. However, recent theoretical progress in~\cite{hron2018variational,hron2017variational,neklyudov2017structured} demonstrates that the log-uniform prior is an improper prior (\ie, its integral is infinite), which causes inferring the posterior of network weights to be ill-posed. Such ill-posed inference can degenerate variational inference on these weights into penalized maximum likelihood estimation~\cite{hron2018variational}. Thus, the interpretation of the regularization capacity of Gaussian dropout is not in a full Bayesian way. And more importantly, the regularization capacity of VD is still limited.

To address this problem, we propose a variational Bayesian dropout (VBD) framework, which is a new generalization of Gaussian dropout. A visual comparison between the proposed VBD and VD can be seen in Figure~\ref{VBD}. In VBD, we assume network weights to come from a two-level hierarchical prior. Instead of only inferring the posterior over network weights, we propose to infer the joint posterior over both network weights and their hyper-parameters defined in their first-level prior. Through implementing the hierarchical prior as a zero-mean Gaussian distribution with variance sampled from a uniform distribution, we can theoretically prove that the proposed VBD can be consistent with Gaussian dropout~\cite{wang2013fast} for a fixed dropout rate as VD. Thus, VBD also can interpret the regularization capacity of Gaussian dropout.
In contrast to the improper log-uniform prior, the proposed hierarchical prior is a proper prior, which enables the inference of posterior in VBD well-posed. This not only leads to a full Bayesian justification for Gaussian dropout, but also improves the regularization capacity obviously. In addition, we further find that the proposed VBD can be seamlessly applied to neural network compression as~\cite{molchanov2017variational,neklyudov2017structured}. Experimental results on classification as well as network compression tasks show the effectiveness of VBD in handling over-fitting.
%-------------------------------------------------------------------------
\section{Related work}
{\textbf{Dropout}}. Dropout plays an important role in improving the generalization capacity of deep neural networks. At first, dropout is employed to randomly drop input features with Bernoulli distribution during training to prevent feature co-adaptation \cite{hinton2012improving}. This amounts to training an exponential number of different networks with shared parameters. In the test phase, the prediction is determined by averaging the outputs of all these different networks. The idea of dropout is then generalized by multiplying the input features with random noise drawn from other distributions, e.g., Gaussian \cite{srivastava2014dropout}. While these early methods have shown effectiveness in some cases, repeatedly dropping a random subset of input features makes training a network much slower. To address this problem, Gaussian dropout \cite{wang2013fast} proposes to sample the output features from a Gaussian distribution instead of input features in dropout training and shows virtually identical regularization performance but faster convergence. This is inspired by the observation that enforcing multiplicative noise on input features, whatever the noise is generated from a Bernoulli distribution or a Gaussian, making use of the central limit theorem, makes the corresponding outputs to be approximately Gaussian~\cite{wang2013fast}. However, these conventional dropout methods fail to clarify the intrinsic principle for their regularization capacity. In addition, their performance depends a lot on the pre-defined dropout rate. In contrast, the proposed VBD can provide a Bayesian interpretation for dropout as well as automatically estimating the dropout rate.

{\textbf{Variational Dropout}}. VD is a generalization of Gaussian dropout, which can interpret the regularization capacity of dropout as well as automatically estimating the dropout rate via inferring the posterior of network weights. For example, literature in~\cite{kingma2015variational} proves that training network with variational dropout framework implicitly imposes the log-uniform prior on weights for preventing over-fitting. Since the dropout rate can be automatically determined, some works of literature ~\cite{molchanov2017variational,neklyudov2017structured} further apply VD to compress neural networks. However, the log-uniform prior is an improper prior which causes the inference of posterior over network weights in VD to be ill-posed~\cite{hron2017variational,neklyudov2017structured}, thus limiting its performance in preventing over-fitting. In this study, the proposed VBD imposes a proper hierarchical prior on network weights, which induces a well-posed Bayesian inference over network weights and thus improves the regularization capacity.

{\textbf{Concrete Dropout and Adversarial Dropout}}. In addition, recent works have proposed another two dropout methods, \eg, concrete dropout~\cite{gal2017concrete} and adversarial dropout~\cite{AAAI1816322}. They are different from the proposed VBD. Specifically, concrete dropout provides Bayesian generalization for dropout with Bernoulli distribution \cite{hinton2012improving}, while the proposed VBD provides Bayesian generalization for Gaussian dropout \cite{wang2013fast}. Besides, adversarial dropout \cite{AAAI1816322} proposes to handle over-fitting by training the network in an adversarial way. In contrast, the proposed method focuses on introducing hierarchical prior on network weights to regularize the network training.

\section{Preliminaries}
Consider a supervised learning problem on a dataset ${\cal D} = {\{ ({ {\mathbf x}_i},{ {\mathbf y}_i})\}}_{i = 1}^N$ of observation-label pairs. We train a fully connected neural network with $L$ hidden layers. For each layer, we have:
\begin{equation}\label{nn}
  {\mathbf{B}} = {{\mathbf A}}{{\mathbf \theta}},
\end{equation}
where ${{\mathbf A}}$ denotes the $M \times K$ matrix of input features for current minibatch, ${{\mathbf \theta}}$ is the $K \times D$ weight matrix, $\mathbf{B}$ is the $M \times D$ output matrix before activation function.

\subsection{Gaussian Dropout}
To prevent over-fitting, dropout applies multiplicative noise on the input of each layer of neural networks during training as follows:
\begin{equation}\label{gd}
  {\mathbf{B}} = ({{\mathbf A} \circ {\mathbf \xi} }){{\mathbf \theta}},
\end{equation}
where ${\mathbf \xi}$ is the $M \times K$ noisy matrix, and $\circ$ denotes the element-wise (Hadamard) product. In conventional dropout methods, the elements of the noise $\xi$ are either sampled from a Bernoulli distribution with probability $1-p$ of being 1, with the {\textit{dropout rate}} $p$ \cite{hinton2012improving}, or sampled from a Gaussian distribution with mean 1 and variance $\alpha ={p}/(1 - p)$~\cite{srivastava2014dropout}. Regardless which strategy of the above two is used, according to the central limit theorem and equation Eq.~\eqref{gd}, one can directly produce $B_{m,d}$ by sampling from the following Gaussian distribution:
\begin{equation}\label{gd1}
q({B}_{m,d}|\mathbf{A}) = {\cal N}({B}_{m,d}| {\mu_{m,d}},{\delta^2_{m,d}}),
\end{equation}
where mean {\small {${\mu_{m,d}} = \sum\nolimits_{k = 1}^K {{A_{m,k}}{\theta_{k,d}}}$ and variance ${\delta^2_{m,d}} = \alpha \sum\nolimits_{k = 1}^K {A^2_{m,k}\theta^2_{k,d}}$}}. Here $A_{m,k}$ denotes an element in $\mathbf{A}$. This means ${q }(\mathbf W) $ can be factorized as follows:
\begin{equation}\label{gd2}
{q }(\mathbf W) = \prod\limits_{k = 1}^K {\prod\limits_{d = 1}^D   q(W_{k,d})} = \prod\limits_{k = 1}^K {\prod\limits_{d = 1}^D {\cal N}(W_{k,d}|\theta_{k,d},{\alpha \theta^2_{k,d}})},
\end{equation}
where each element $W_{k,d}$ in $\mathbf{W}$ can be sampled from ${q }(\mathbf W_{k,d}) $ in Eq.~\eqref{gd2}. Finally, the objective function for network training with Gaussian dropout becomes:
\begin{equation}\label{vd1}
\mathop {\max }\limits_{\theta} {L_{\cal D}}(\theta ) \simeq \mathop {\max }\limits_{\theta} \sum\limits_{i = 1}^N {\mathbb{E}_{q(\mathbf W)}} {\log p({{\mathbf y}_i}|{{\mathbf x}_i}, {\mathbf W} )}.
\end{equation}
The viewpoint in Eq.~\eqref{vd1} provides a opportunity to bridge the gap between Bayesian inference and dropout, if we use Eq.~\eqref{gd2} as an approximate posterior distribution for a network model with a special prior on the weights. The challenge in this gap is what is the special prior.

\subsection{Variational Dropout}
VD uses $q(\mathbf W)$ in Eq.~\eqref{gd2} as a variational posterior to approximate the true posterior $p(\mathbf W|{\cal D})$ in terms of minimal KL divergence,
\begin{equation}\label{VarDrop}
 \min\limits_{\theta,\alpha} D_{KL}({q }(\mathbf W)||p(\mathbf W|{\cal D})).
\end{equation}
Here $D_{KL}(\cdot)$ denotes the KL divergence.
Given $\mathbf{W}$ from a prior $p(\mathbf{W})$, according to the Bayesian rule, i.e., $p(\mathbf W|{\cal D})\propto p({\cal D}|\mathbf W)p({\mathbf W})$, minimizing the KL divergence in Eq.~\eqref{VarDrop} is equivalent to maximizing the variational lower bound of the marginal likelihood of data as:
\begin{equation}\label{vd}
\mathop {\max }\limits_{\theta,\alpha} {L_{\cal D}}(\theta,\alpha ) - {D_{KL}}({q }(\mathbf W)||p(\mathbf W)),
\end{equation}
where ${L_{\cal D}}(\theta,\alpha)$ with fixed parameter $\alpha$ is the same as one in Eq. \eqref{vd1} and known as the expected log-likelihood term.

% ----TODO
Firstly, in VD, the dropout rate $\alpha$ can be automatically determined by data characteristics. Secondly, VD can provide a Bayesian interpretation for the success of dropout in preventing over-fitting. To clarify this, VD requires that the optimization of $\theta$ in Eq.~\eqref{vd} is consistent with that in Gaussian dropout in Eq.~\eqref{vd1}, i.e., maximizing the expected log-likelihood. To this end, the prior $p(\mathbf{W})$ has to be such that ${D_{KL}}({q }(\mathbf W)||p(\mathbf W))$ in Eq. \eqref{vd} does not depend on weight parameters $\theta$.
With such a requirement, we have the following proposition.
\begin{prop}[\cite{kingma2015variational}]
\label{prop1}
The only prior in VD, which enables  ${D_{KL}}({q }(\mathbf W)||p(\mathbf W))$ not depending on weight parameters $\theta$, is the log-uniform prior:
\begin{equation}
\label{up}
p(|W_{k,d}|) = \frac{c}{{|W_{k,d}|}} \Leftrightarrow p(\log (|W_{k,d}|)) = c.
\end{equation}
\end{prop}

The above discussion demonstrates that training network with VD implicitly imposes the log-uniform prior on weights. With such a prior, the KL term in Eq.~\eqref{vd} is able to regularize the number of significant digits stored for the weights $\mathbf{W}$ in the floating-point format, thus being able to mitigate over-fitting at some extent. However, the log-uniform prior is an improper prior, which reaches ill-posed variational inference, e.g., the KL divergence between the variational posterior in Eq.~\eqref{gd2} and the log uniform prior in Eq.~\eqref{up} is infinite. Although leveraging truncated techniques relieves this problem at some extent \cite{kingma2015variational,neklyudov2017structured}, pathological behaviour still remains \cite{hron2018variational}, e.g., the resultant infinite KL divergence theoretically degenerate VD into a maximum likelihood estimation that fails to avoid over-fitting. More details for the theoretical justification can be found in~\cite{hron2017variational,neklyudov2017structured,hron2018variational}. Therefore, the performance of VD in preventing over-fitting need to be further improved.

In this section, we will introduce the details of the proposed VBD. In the following, we first introduce the proposed VBD framework. Then, we show a specifically designed hierarchical prior in VBD framework and prove that VBD with this prior can be consistent with Gaussian dropout for a fixed dropout rate.
\subsection{Variational Bayesian Dropout}
In contrast to a one-level prior in VD, we turn to propose a two-level hierarchical prior $p(\mathbf W,\gamma)=p(\mathbf W|\gamma)p(\gamma)$. This brings two aspects of advantages. Firstly, two kinds of very simple distributions in hierarchical structure can produce a much more complicated distribution, e.g., a hierarchical sparse prior~\cite{zhang2018cluster}, a zero-mean Gaussian distribution with variance depicted by Gamma distribution \cite{zhang2015reweighted}, and the super-Gaussian scale mixture model \cite{liu2018deblurring}. Thus the two-level structure increase the possible solution spaces for the proper and feasible prior to interpret Gaussian dropout. Secondly, the hierarchical structure enables the two-level prior separable in the involved Bayesian inference and thus is possible to simplify the Bayesian inference or makes the intractable inference tractable, which will be further clarified in the following subsections.

Similar to VD, the proposed VBD aims at optimizing a variational posterior to approximate the true posterior \cite{jordan1999introduction,Liu2018Frame}.
% They differ in that VD only considers the posterior of network weights, while we propose to model the joint posterior of both weights (e.g., $\mathbf{W}$) and the hyper-parameters (e.g., $\gamma$) in their prior as illustrated in Figure~\ref{VBD}.
Unlike that VD only considers the posterior of network weights, we propose to model the joint posterior of both the network weights (e.g., $\mathbf{W}$) and the hyper-parameters (e.g., $\gamma$) in their prior as illustrated in Figure~\ref{VBD}.
% Therefore, the objective of Bayesian inference in the proposed VBD can be formulated as:
We thus arrive the objective of Bayesian inference in the proposed VBD:
\begin{equation}\label{VarBaDrop}
 \min\limits_{\theta, \alpha, \gamma} D_{KL}({q }(\mathbf W, \gamma)||p(\mathbf W, \gamma|{\cal D})),
\end{equation}
where ${q }(\mathbf W, \gamma)$ denotes a corresponding variational joint posterior for $p(\mathbf W, \gamma|{\cal D})$. Note that when the hyper-parameter $\gamma$ is fixed, the proposed VBD will reduce to VD. Thus, the proposed VBD is a more general version of VD. According to variational Bayesian inference technique~\cite{jordan1999introduction}, we use the variational posterior $q(\mathbf W,\gamma)=q(\mathbf W)q(\gamma)$ to approximate the true posterior $p(\mathbf W, \gamma|{\cal D})$, and then the objective in Eq.~\eqref{VarBaDrop} can be reformulated as the variational lower bound of the marginal likelihood of data as:
\begin{equation}\label{vdh}
%\begin{split}
\mathop {\max }\limits_{\alpha, \theta,\gamma} {L_{\cal D}}(\alpha, \theta) - {D_{KL}}({q }(\mathbf W)||p(\mathbf W| {{\gamma}}))- {D_{KL}}({q }({{\gamma}})||p({{\gamma}})),
%\end{split}
\end{equation}
where ${L_{\cal D}}$ is the expected log-likelihood term in Eq.~\eqref{vd}.
Derivations can be found in supplementary material.
For the proposed VBD, the key is to exploit a proper hierarchical prior $p(\mathbf{W}, \gamma)$ for supporting Gaussian dropout. In the following, we will provide such prior and discuss its advantages.

\subsection{The Proposed Hierarchical Prior}
\label{sec:g}
Inspired by the hierarchical prior in sparse Bayesian learning~\cite{tipping2001sparse}, we assume the network weights $\mathbf{W}$ come from a zero-mean Gaussian distribution. Then, a uniform hyper-prior is imposed on the variance of the Gaussian distribution to adjust the shape of the ultimate prior. When each element $W_{k,d}$ in $\mathbf{W}$ is independent identically distributed, the proposed two-level hierarchical prior $p(\mathbf{W}, \gamma)$ is formulated as:
\begin{equation}\label{iidgp}
\begin{split}
p(\mathbf W| \gamma) &= \prod\limits_{k = 1}^K {\prod\limits_{d = 1}^D p(W_{k,d}|\gamma_{k,d})}= \prod\limits_{k = 1}^K {\prod\limits_{d = 1}^D {\cal N}(W_{k,d}|0,\gamma_{k,d})},\\
p(\gamma) &= \prod\limits_{k = 1}^K {\prod\limits_{d = 1}^D p(\gamma_{k,d})} = \prod\limits_{k = 1}^K {\prod\limits_{d = 1}^D {\mathcal {U}}(\gamma_{k,d}|a,b)},
\end{split}
\end{equation}
where ${\mathcal {U}}(\gamma_{k,d}|a,b)$ denotes an uniform distribution with range $[a,b]$. By imbedding this prior into the proposed VBD framework in Eq.~\eqref{vdh}, we give the following theoretical result.
\begin{prop}
\label{theory:VBD}
With the prior $p(\gamma)$ in Eq.~\eqref{iidgp}, we employ mean-field variational approximation, viz., $q(\gamma)= \prod\nolimits_{k = 1}^K {\prod\nolimits_{d = 1}^D q(\gamma_{k,d})}$, and assume that $q(\gamma_{k,d})$ comes from a delta distribution. Then the proposed VBD framework in Eq.~\eqref{vdh} reduces to
\begin{equation}\label{vdh1}
\mathop {\max }\limits_{\alpha, \theta,\gamma} {L_{\cal D}}(\alpha, \theta) - {D_{KL}}({q }(\mathbf W)||p(\mathbf W| {{\gamma}})).
\end{equation}
\end{prop}
Given the prior in Eq. \eqref{iidgp} and the variational posterior $q(\gamma_{k,d})$, ${D_{KL}}({q }({{\gamma}})||p({{\gamma}}))$ in Eq.~\eqref{vdh} will collapse to a constant and thus can be neglected in optimization.
Similar trick can be found in \cite{chan2002variational,babacan2012bayesian}.
Specifically, to simplify the representation, we assume $\gamma_{k,d}$ as a one-dimensional scalar. As defined in Eq.~\eqref{iidgp}, we have $p(\gamma_{k,d}) = {1}/{ (b - a)}$. Note that the delta distribution $q(\gamma_{k,d})$ either lies in or out of $[a, b]$, e.g., $\delta(\gamma_{k,d}-\gamma'_{k,d})$ and $\gamma'_{k,d}$ lies in or out of $[a,b]$. If $q(\gamma_{k,d})$ is out of $[a,b]$, there is $D_{KL}(q(\gamma_{k,d})||p(\gamma_{k,d})) = +\infty$. To avoid this case, $[a,b]$ is generally regarded as a large enough interval \cite{chan2002variational,babacan2012bayesian}. As a result, we arrive $D_{KL}(q(\gamma_{k,d})||p(\gamma_{k,d})) = \log (b - a)$, which is independent to the unknown variables $\alpha, \theta$ and $\gamma$ and thus can be neglected. Therefore, we do not need to set specific values for the hyper-parameters $a$ and $b$ in practice. The detailed proof can be found in supplementary material.

According to~\cite{kingma2015variational}, the key property of the log-uniform prior is to enable the KL divergence ${D_{KL}}({q }(\mathbf W)||p(\mathbf W))$ in Eq.~\eqref{vd} not depending on weight parameters $\theta$ as mentioned in {\bf{Proposition 1}}. With this property, learning $\theta$ in VD will be consistent with that in conventional Gaussian dropout for a fixed dropout rate $\alpha$. In the following, we will demonstrate that the proposed hierarchical prior Eq.~\eqref{iidgp} also shows a similar property in VBD framework Eq.~\eqref{vdh1}. To this end, we give the following theoretical result.

\begin{prop}
\label{theory}
Given the objective Eq.~\eqref{vdh1}, together with the prior Eq.~\eqref{iidgp} and the variational posterior Eq.~\eqref{gd2}, ${D_{KL}}({q }(\mathbf W)||p(\mathbf W|\gamma))= {\sum\nolimits_{k = 1}^K {\sum\nolimits_{d = 1}^D 0.5\log (1 + {\alpha^{-1}})}}$, which does not depend on weight parameters $\theta$.
\end{prop}

\begin{proof}
Since the variational posterior in Eq.~\eqref{gd2} and the prior $p(\mathbf W| \gamma)$ in Eq.~\eqref{iidgp} are fully factorized, the KL-divergence ${D_{KL}}({q }(\mathbf W)||p(\mathbf W| {{\gamma}}))$ in \eqref{vdh1} can be decomposed into a sum as:
\begin{equation}\label{pf11}
{\sum\limits_{k = 1}^K {\sum\limits_{d = 1}^D {D_{KL}}({q}( W_{k,d})||p( W_{k,d}|\gamma_{k,d})) }}.
\end{equation}
Since both the prior ${q}( W_{k,d})$ and the posterior $p( W_{k,d}|\gamma_{k,d})$ follow Gaussian distributions, the KL divergence ${D_{KL}}({q }( W_{k,d})||p( W_{k,d}|\gamma))$ in Eq. \eqref{pf11} can be calculated as:
\begin{equation}\label{pf1}
\begin{split}
{D_{KL}}({q }( &W_{k,d})||p( W_{k,d}|\gamma))= \\
&{ { 0.5 \log (\frac{{\gamma_{k,d}}}{{\alpha {\theta}^2_{k,d}}}) + \frac{{\alpha {\theta}^2_{k,d}} +{\theta }^2_{k,d}  }{{2\gamma _{k,d}}}}} - 0.5.
\end{split}
\end{equation}
By introducing Eq.~\eqref{pf1} into Eq.~\eqref{vdh1}, we arrive:
\begin{small}
\begin{equation}\label{pf2}
\mathop {\max }\limits_{\alpha,\theta,\gamma} {L_{\cal D}}(\alpha,\theta)
  -  {\sum\limits_{k = 1}^K {\sum\limits_{d = 1}^D 0.5\log (\frac{{\gamma _{k,d}}}{{\alpha {\theta}^2_{k,d}}}) + \frac{{\alpha {\theta}^2_{k,d}} +{\theta}^2_{k,d}  }{{2\gamma _{k,d}}}}}.
\end{equation}
\end{small}
To find the optimal ${\gamma_{k,d}}$, referred to as ${\gamma^*_{k,d}}$, by setting the partial differential of Eq.~\eqref{pf2} with respect to ${\gamma _{k,d}}$ to zero, we have:
\begin{equation}\label{pf3}
{\gamma^*_{k,d}} = {\alpha {\theta}^2_{k,d}} +{\theta}^2_{k,d}.
\end{equation}
Replacing ${\gamma _{k,d}}$ in Eq.~\eqref{pf1} by Eq.~\eqref{pf3} completes the proof.
\end{proof}
In summary, with {\bf{Propositions} \ref{theory:VBD}} and {\bf{\ref{theory}}}, the final objective for the proposed VBD with the hierarchical prior can be given as:
\begin{equation}\label{pf4}
\mathop {\max }\limits_{\alpha, {\theta}} {L_{\cal D}}(\alpha, {\theta} )
- {\sum\limits_{k = 1}^K {\sum\limits_{d = 1}^D 0.5 \log (1 + {\alpha^{-1}})}}.
\end{equation}
Built on this loss function, we can see that: 1) The proposed hierarchical prior also meets the requirement such that ${D_{KL}}({q }(\mathbf W)||p(\mathbf W|\gamma))$ does not depend on weight parameters $\theta$. In other words, the proposed VBD with the hierarchical prior is consistent with Gaussian dropout when $\alpha$ is fixed. Hence, the proposed VBD can give Bayesian interpretation for Gaussian dropout. 2) The dropout rate parameter $\alpha$ also can be automatically learned as VD.

Note that we apply a uniform prior on $\gamma$ to be able to update $\gamma$. If $\gamma$ was treated as the variance hyper-parameter of the Gaussian prior, we could not update it in the proposed framework. This is because the prior could not see any data. Then, if we cannot update $\gamma$, we cannot produce Eq.~\eqref{pf3}, which means that we could not produce {\bf Propositions} \ref{theory}. To allow $\gamma$ to be updated, we impose a uniform prior on it and utilize delta variational posterior leading to {\bf Propositions 2}. This is also why many variational approximation methods (e.g., \cite{chan2002variational,babacan2012bayesian}) employ priors on hyper-parameters.

More importantly, the improper log-uniform prior in traditional VD induces ill-posed Bayesian inference, since it leads to infinite ${D_{KL}}({q }(\mathbf W)||p(\mathbf W))$ as mentioned in \cite{hron2018variational,hron2017variational,neklyudov2017structured}. By contrast, the proposed hierarchical prior in VBD gives well-posed Bayesian inference, since it produces reasonable and tractable ${D_{KL}}({q }(\mathbf W)||p(\mathbf W|\gamma))$ as shown in {\bf{Proposition} \ref{theory}}. Hence, the proposed hierarchical prior in VBD shows obvious superiority over the log-uniform prior in VD, and we argue that the proposed VBD framework with the hierarchical prior provides a full Bayesian interpretation for the success of Gaussian dropout in preventing over-fitting.

According to the discussion above, we can conclude as follows. 1) The proposed VBD is a more general VBD. VD focuses on incorporating one-level priors, while VBD can contain two-level priors as shown in Figure~\ref{VBD}. 2) VD claims that the improper log-uniform prior interprets Gaussian dropout, which fails to give full Bayesian interpretation for Gaussian dropout and suffers limited regularization capacity, while VBD proposes a proper prior to handle those drawbacks. In the following, we will show that these two main differences further lead to the apparent advantage of VBD in network compression compared with VD.

\section{Extension to Neural Network Compression}
Since VD can adaptively learn the dropout rate $p$ (or $\alpha$) from training dataset, it can be utilized for neural networks compression \cite{molchanov2017variational,neklyudov2017structured}. Inspired by this, in this section we turn to exploit the ability to apply the proposed VBD to neural networks compression under the frameworks proposed in \cite{molchanov2017variational,neklyudov2017structured}. In addition, we will also discuss the advantage of the proposed framework on networks compression.

\subsection{Compressing Weights}
We first extend the proposed VBD to compressing weights under the framework in \cite{molchanov2017variational}. Further details about the framework can be found in that paper. To this end, the proposed hierarchical prior Eq.~\eqref{iidgp} is used to model weights in a neural network.
For convenience, we replace the original $\alpha$ in Eq.~\eqref{gd2} to learn specific $\alpha$ for each weight, in which {\bf{Proposition} \ref{theory}} holds.
The distribution Eq.~\eqref{gd2} as a variational posterior, ${\cal N}(\theta_{k,d},{\alpha_{k,d} \theta^2_{k,d}})$, is used to approximate the true posterior. In this way, to learn $\theta_{k,d}$ and $\alpha_{k,d}$, the objective function in Eq.~\eqref{pf4} is rewritten as:
\begin{equation}\label{obj}
\mathop {\max }\limits_{\alpha, {\theta}} {L_{\cal D}}(\alpha, {\theta} )
- {\sum\limits_{k = 1}^K {\sum\limits_{d = 1}^D 0.5 \log (1 + {\alpha^{-1}_{k,d}})}}.
\end{equation}
Furthermore, since the natural gradient of $\theta_{k,d}$ faces with high variance, we follow the re-parameterization trick in \cite{molchanov2017variational}, viz, $\sigma^2_{k,d}={\alpha_{k,d} \theta^2_{k,d}}$.

The difference between the proposed method and the method in \cite{molchanov2017variational} is on the regularization term $-{D_{KL}}({q }(\mathbf W)||p(\mathbf W))$. It is equivalent to ${\sum\nolimits_{k = 1}^K {\sum\nolimits_{d = 1}^D -0.5 \log (1 + \alpha^{-1}_{k,d} )}}$ in the proposed VBD, while it is {\small $ {\sum\nolimits_{k = 1}^K {\sum\nolimits_{d = 1}^D {k_1}S({k_2} + {k_3}\log (\alpha_{k,d} )) \\ -0.5 \log (1 + \alpha^{-1}_{k,d}   )}}$} in the method proposed in \cite{molchanov2017variational}, where $k_1$, $k_2$ and $k_3$ are constant, and $S(\cdot)$ denotes the sigmoid function. The term $- 0.5 \log (1 + \alpha^{-1}_{k,d})$ in the latter is heuristically designed in~\cite{molchanov2017variational} to model the behaviour that the negative KL-divergence goes to a constant as log $\alpha_{k,d}$ goes to minus infinity. In contrast, the term $- 0.5 \log (1 + \alpha^{-1}_{k,d})$ in the proposed VBD is naturally derived from Bayesian inference as mentioned above. Further discussion on this term will be given in section \ref{analysis}.

\subsection{Structured Compressing}
Although the method in \cite{molchanov2017variational} can be employed to compress weights in a neural network, it fails to accelerate neural networks in the testing phase, since resultant compression is unstructured. Recently, structured Bayesian pruning in~\cite{neklyudov2017structured} employs VD to remove neurons and/or convolutional channels in convolutional neural networks for structured compression, resulting in satisfactory performance. Similar to VD, the proposed VBD also can be employed for structured pruning by constructing a dropout-like layer under the framework in \cite{neklyudov2017structured}.

Specifically, we construct a single dropout-like layer with an input matrix $f(\mathbf B)$ as follows:
\begin{equation}\label{gd12}
  \mathbf B' = f(\mathbf B) \circ {\mathbf W' },
\end{equation}
where ${\mathbf W'}$ denotes the dropout noise and $f(\cdot)$ denotes the activation function. The output of this layer $\mathbf B'$ is of the same size as the input $\mathbf B$, and would serve as an input matrix for the following layer. Similar to that in the previous section, we enforce the proposed hierarchical prior Eq.~\eqref{iidgp} on ${\mathbf W' }$, and imposes the variational posterior, ${\cal N}(\theta_{m,d},{\alpha_{m,d} \theta^2_{m,d}})$ (For convenience, we re-utilize the same symbol $\alpha_{m,d}$ and $\theta_{m,d}$ that is originally used to model weights in the previous section, however, they are served to $\mathbf W'$ in this section). Again, the objective function Eq.~\eqref{obj} is used for learning $\theta_{m,d}$ and $\alpha_{m,d}$, and the re-parameterization \cite{molchanov2017variational} is adopted.

\subsection{Analysis}
\label{analysis}
In these two kinds of compression schemes above, effective compression depends on high dropout rate, e.g., $\alpha_{k,d} \to \infty$ or $\alpha^{-1}_{k,d} \to 0$, which corresponds to a binary dropout rate that approaches $p = 1$. This effectively means that the corresponding weight or neuron is always ignored and can be removed \cite{molchanov2017variational}. In this subsection, we will show that the proposed variational Bayesian dropout explicitly imposes a sparse regularization for optimizing $\alpha^{-1}_{k,d}$, and thus is able to effectively compress the deep neural networks. To this end, we firstly rewrite the objective Eq.~\eqref{obj} as:
\begin{equation}\label{fobj}
\mathop {\min }\limits_{\alpha, {\theta}} {\sum\limits_{k = 1}^K {\sum\limits_{d = 1}^D 0.5 \log (1 + {\alpha^{-1}_{k,d}})} -L_{\cal D}(\alpha, {\theta} ) }.
\end{equation}
The expected log-likelihood term $L_{\cal D}$ can be viewed as the data fit-term for $\alpha^{-1}_{k,d}$ and the remainder derived from KL divergence works as the regularization term. For such a regularization term, we have the following theoretical results.
\begin{prop}\label{corollary}
The regularization term $0.5 \log (1 + {\alpha^{-1}_{k,d}})$ in Eq.~\eqref{fobj} is a concave, non-decreasing function on the domain $[0, +\infty)$, with respect to ${\alpha^{-1}_{k,d}}$.
\end{prop}
\noindent According to \cite{pmlrchen17d}, introducing such a regularization term $0.5 \log (1 + {\alpha^{-1}_{k,d}})$ into the objective is beneficial to promote the sparsity of the solution. Therefore, with optimizing $\alpha^{-1}_{k,d}$, we can obtain sparse $\alpha^{-1}_{k,d}$. Note that this regularization term coincides with that in \cite{dai18d} which is motivated by information bottleneck principle. In contrast, the regularization term in this study stems from variational Bayesian inference. Besides, this sparsity-promoting regularization comes from a result of the particular variational approximation, which is different from previous methods with sparse priors to compress network, e.g., \cite{louizos2017bayesian,ghosh2018structured}.
\section{Experiments}
In this section, we conduct experiments on classification task to demonstrate the effectiveness of the proposed variational Bayesian dropout in preventing over-fitting. Then, we further evaluate its performance in neural network compression including weight compression and structured compression. Note that neural network compression can also imply the ability of preventing over-fitting in term of the final test error.
\subsection{Classification}
{\textbf{MNIST Dataset}} Following the settings in~\cite{kingma2015variational}, we first take the hand-written digit classification task on MNIST dataset as a standard benchmark to evaluate the performance of dropout methods in preventing over-fitting. On this task, we compare the proposed variational Bayesian dropout with other five existing dropout methods, namely no dropout, standard dropout with Bernoulli noise \cite{hinton2012improving}, dropout with Gaussian noise \cite{srivastava2014dropout}, Gaussian dropout \cite{wang2013fast} and VD \cite{kingma2015variational}, and concrete dropout \cite{gal2017concrete} which is able to learn adaptive dropout rate. We follow the network architecture in~\cite{srivastava2014dropout}, which adopt a fully connected neural network consisting of 3 hidden layers with different number of units and rectified linear units. For experimental setting, all networks are trained for 50 epochs. More details about the network architecture can be found in supplementary material.

\begin{table}[!ht]
\center
\small
\begin{tabular}{c|c|c|c|c|c}
\hline
  {\footnotesize{ The number of units}}& 100 & 340 &580 & 820 & 1060  \\
  % Methods & Error (\%)  \\
\hline\hline
  No Dropout                & 1.80 & 1.77 & 1.78 & 1.78 & 1.71  \\
  Dropout, Bernoulli        & 1.69 & 1.68 & 1.67 & 1.67 & 1.63  \\
  Dropout, Gaussian         & 1.70 & 1.66 & 1.68 & 1.69 & 1.62  \\
  Gaussian Dropout          & 1.66 & 1.64 &1.71  & 1.65 & 1.64   \\
  Concrete Dropout          & 2.39 & 1.61 & 1.55 & 1.54 & 1.51 \\
   VD                      &  1.69 & 1.62 & 1.67 &1.62 & 1.63     \\
   Ours   & {\bf{1.56}}& {\bf{1.53}}& {\bf{1.53}}& {\bf{1.52}}& {\bf{1.45}}   \\
\hline
\end{tabular}
\caption{Test error (\%) on the MNIST dataset. \label{dropouta} }
\end{table}

Table \ref{dropouta} shows the test error for all methods with various choices of the number of units per layer. We observe that although VD can adaptively learn dropout rate during training, it only obtains slightly better performance compared with conventional dropout methods with fixed dropout rate, including dropout with Bernoulli noise and dropout with Gaussian noise. This is because the Bayesian inference in VD with the improper log-uniform prior is ill-posed, which hence ultimately restricts the capacity in preventing over-fitting as discussed before. Conversely, the proposed VBD with adaptive dropout rate gains impressive performance, which is better than that of the standard dropout as well as VD. This profits from that the proposed hierarchical prior is a proper prior and thus it can appropriately regularize the network weights in the Bayesian inference. Besides, the proposed VBD is superior to concrete dropout.

{\textbf{CIFAR-10 Dataset}} We further compare the proposed method with dropout with Bernoulli noise \cite{hinton2012improving}, dropout with Gaussian noise \cite{srivastava2014dropout}, VD \cite{kingma2015variational}, concrete dropout \cite{gal2017concrete} on CIFAR-10 dataset. We also compare with adversarial dropout recently proposed in \cite{AAAI1816322}. In this case, we follow the network architecture with different $scale$ in~\cite{kingma2015variational} for experimental setting, all networks are trained for 100 epochs. More details for the network architecture can be found in supplementary material.
\begin{table}[!ht]
\center
\begin{tabular}{c|c|c|c}
\hline
  $scale$           & 1 & 1.5 & 2 \\
  % Methods & Error (\%)  \\
\hline\hline
  No Dropout                & 48.34 & 48.13 & 47.56  \\
  Dropout, Bernoulli        & 45.55 & 43.38 & 42.69  \\
  Dropout, Gaussian         & 46.52 & 43.50 & 42.79  \\
  Adversarial Dropout       & 45.50 & 42.35 & 42.50  \\
  Concrete Dropout          & 43.47 & 42.68 & 42.19  \\
  VD                & 44.12 & 42.99 & 42.60  \\
  Ours   &        {\bf{39.50}}& {\bf{38.88}}& {\bf{38.52}} \\
\hline
\end{tabular}
\caption{Test error (\%) on the CIFAR-10 dataset. \label{dropoutb} }
\end{table}

Figure \ref{dropoutb} shows the test error for the all methods with different $scale$. We can see that concrete dropout only performs on par with traditional dropout methods and VD. In addition, we found that due to its negative effect of the improper log-uniform prior, VD only provides comparable results to those methods with fixed dropout rate. In contrast, profiting from the proper hierarchical prior, the proposed variational Bayesian dropout performs impressively well in preventing over-fitting. For example, when $scale=2$, compared with no dropout method, the proposed method reduces the test error by 9.04\%. The improvement is even up to 4.08\% when compared with VD.

{\textbf{SVHN Dataset}} In this case, we follow the network architecture with $scale=1$ used in experiments on {\textbf{CIFAR-10 Dataset}}, and all networks are trained for 100 epochs Table \ref{tab:svhn} shows the results of test error (\%) in SVHN dataset. Since traditional methods for setting dropout rate, such as Grid-search based methods, are computationally expensive. We set the dropout rate to be 0.5 for all layers of the network in our experiments for simplicity. Under this simple setting, traditional dropouts, e.g., Dropout with Bernoulli noise and Dropout with Gaussian noise, are slightly superior to no dropout. Again, due to the improper log-uniform prior, VD only provides comparable results to traditional dropout with fixed dropout rate. Conversely, compared with these fixed dropout rate based methods, the proposed method with adaptive dropout rate gets the best test error, 17.46\%.

\begin{table}[t]
\center
\begin{tabular}{c|c}
\hline
   Methods & Error (\%)  \\
\hline\hline
  No dropout & 22.01  \\
  Dropout, Bernoulli &  20.31  \\
  Dropout, Gaussian   & 20.22   \\
  Concrete Dropout     & 18.95  \\
  Adversarial Dropout   & 18.66    \\
   VD    &19.74     \\
    Ours   & {\bf{17.46}}  \\
\hline
\end{tabular}
\caption{Test error on the SVHN dataset. \label{tab:svhn} }
\end{table}

\subsection{Network compression}
{\textbf{Compressing Weights}} In this part, we turn to evaluate the effectiveness of the proposed method in weight compression in neural networks for classification on the MNIST dataset. Here we adopt two kinds of basic neural networks, the fully-connected LeNet-300-100 \cite{lecun1998gradient} and a convolutional LeNet-5-Caffe \footnote{https://github.com/BVLC/caffe/tree/master/examples/mnist}. We compare the proposed method with four network compression methods, including Pruning \cite{han2015learning}, Dynamic Network Surgery (DNS) \cite{guo2016dynamic}, Soft Weight Sharing (SWS) \cite{ullrich2017soft} and VD \cite{molchanov2017variational}. In the experiments, we strictly follow the settings in~\cite{molchanov2017variational}.

\begin{table*}[t]
\centering
\begin{tabular}{c|c|c|c||c|c|c}
\hline
  &\multicolumn{3}{c||}{LeNet-300-100} & \multicolumn{3}{c}{LeNet-5-Caffe}  \\
\hline\hline
   Methods & Error \% & Sparsity per Layer \% & $\frac{|\mathbf W|}{|\mathbf W_{\neq 0}|}$  & Error \% & Sparsity per Layer \% & $\frac{|\mathbf W|}{|\mathbf W_{\neq 0}|}$\\
\hline\hline
  Original & 1.64 &              & 1          &           0.8 &              & 1\\
  Pruning  & {\bf{1.59}} & 92.0-91.0-74.0& 12 &           0.77 & 34-88-92.0-81 & 12 \\
    DNS    & 1.99 & 98.2-98.2-\bf{94.5}& 56   &           0.91 & {\bf{86}}-97-99.3-96 & 111\\
    SWS    & 1.94 &               & 23        &            0.97 &               & 200\\
    VD     & 1.94 & {\bf{98.9}}-97.2-62.0& 68 &         {0.75} & 67-{\bf{98}}-{\bf{99.8}}-95 & 280\\
    Ours (low test error)   & 1.67 & 98.7-97.4-87.4& 66        &         {\bf{0.74}} & 69-{\bf{98}}-99.7-95 & 198\\
    Ours (high sparsity)   & 1.76 & {\bf{98.9}}-{\bf{98.1}}-90.9& \bf{81}  & 0.81 & 65-\bf{98}-\bf{99.8}-\bf{97}&  \bf{290}\\
\hline
\end{tabular}
\caption{Compressing weights in LeNet-300-100 and LeNet-5-Caffe. $|\mathbf W|$ and $|\mathbf W_{\neq 0}|$ denote the number of weights, the number of weights with nonzero value, respectively. The sparsity per layer is computed by the rate between the number of weights with zero value and the number of weights in each layer.\label{tab:weight} }
\end{table*}

Table \ref{tab:weight} shows the results of compressing weights in LeNet-300-100. Compared with traditional VD, with the similar sparsity $\frac{|\mathbf W|}{|\mathbf W_{\neq 0}|}=66$, the proposed method gets 1.67\% test error that is better than the result of traditional VD 1.94\%.  Further, although traditional VD has better results in compression ratio $\frac{|\mathbf W|}{|\mathbf W_{\neq 0}|}= 68$, compared with Pruning, DNS and SWS, it also reports higher test error 1.94\% due to its limited ability to avoid over-fitting. On the contrary, since the proposed VBD provides a real Bayesian interpretation for dropout, it can effectively prevent over-fitting and gains the better test error 1.76\% as well as higher compression ratio $\frac{|\mathbf W|}{|\mathbf W_{\neq 0}|}=81$.

Table \ref{tab:weight} also shows the results of compressing weights in LeNet-5-Caffe. We can see that compared with Pruning, DNS and SWS, traditional VD reports better compression ratio $\frac{|\mathbf W|}{|\mathbf W_{\neq 0}|}=280$ and test error 0.75\%. Further, compared with the all methods, the proposed method obtains the best compression ratio $\frac{|\mathbf W|}{|\mathbf W_{\neq 0}|}$, without the loss of test error.

{\textbf{Structured Compressing}} We here test the performances of the proposed method on structured compressing for neural networks. The used architecture of neural networks is a fully-connected LeNet-500-300 and a convolutional LeNet-5-Caffe. %\footnote{A modified version of LeNet5 from \cite{lecun1998gradient}. Caffe Model specification: https://goo.gl/4yI3dL}.
We compare VD \cite{molchanov2017variational}, SSL \cite{wen2016learning} and SBP \cite{neklyudov2017structured}.

\begin{table}[!htbp]
\centering
\footnotesize
 \caption{Compressing neurons. \label{str:weight} }
 \begin{tabular}{c||c|c|c}
\hline
Models & Methods & Error \% & Neurons per Layer   \\
\hline\hline
\multirow{5}{*}{LeNet-500-300} &Original & 1.54 &    784 - 500 - 300 - 10   \\
                    &VD   & 1.57 &    537 - 217 - 130 - 10   \\
                    &SSL    & 1.49 &    434 - 174 - 78\,\,\, - 10    \\
                    &SBP    & 1.55 &    245 - {\bf{160}} - {\bf{55}}\,\,\, - 10    \\
                    &Ours   & {\bf{1.35}} &    {\bf{179}} - {\bf{160}} - 60\,\,\, - 10    \\
\hline\hline
 \multirow{5}{*}{LeNet-5-Caffe} &Original & 0.80 & 20 - 50 - 800 - 500 \\
                                &VD     & 0.75  & 17 - 32 - 329 - 75\,\,\, \\
                                &SSL    & 1.00 & {\bf{3\,\,\,}} - {\bf{12}} - 800 - 500  \\
                                &SBP     & 0.86& {\bf{3\,\,\,}} - 18 - 284 - 283 \\
                                &Ours   & \bf{0.66} & 16 - 34 - {\bf{123}} - {\bf{62}\,\,\,}     \\
\hline
 \end{tabular}
\end{table}
Table \ref{str:weight} shows the results of compressing neurons in LeNet-500-300. As discussed in \cite{trippe2018overpruning,hron2018variational}, VD based SBP corresponds to maximum likelihood estimation, which leads to overly pruning neurons, and hence SBP gets the higher test error 1.55\%. The proposed method not only prunes the most neurons but also gains the lowest test error. By contrast, due to the improper log-uniform prior, VD produces the highest test error and only slightly compresses neurons.

Table \ref{str:weight} also shows the results of neurons compression in LeNet-5-Caffe for all methods. We find that the proposed method gains the lowest test error as well as the least neurons, e.g., the test error of the proposed method is even up to 0.66\%. In addition, in LeNet-5-Caffe, the first two layers are convolutional layers, and the following two layers are fully-connected layers. Differing from SSL and SBP that mainly focus on pruning neurons in convolutional layers, the result shows that the proposed method prefers to pruning neurons in the fully-connected layers. This means that the proposed method emphasizes feature extraction.

%------------------------------------------------------------------------
\section{Conclusion}
In this study, we propose a new generalization (\ie, VBD) for Gaussian dropout to address the drawback of VD brought by the improper log-uniform prior, e.g., the ill-posed inference of posterior over network weights.
Towards this goal, we exploit a hierarchical prior to the network weights and propose to infer the joint posterior over both these weights and the hyper-parameters defined in their first-level prior. Through implementing the hierarchical prior as a zero-mean Gaussian distribution with variance sampled from a uniform hyper-prior, the proposed VDB can cast the network training and the dropout rate estimation into a joint optimization problem. In VBD, the hierarchical prior is a proper prior which enables the inference of posterior to be well-posed, thus not only leading to a full Bayesian justification for Gaussian dropout but also improving regularization capacity. In addition, we also show that the proposed VBD can be seamlessly applied to network compression. In the experiments on both classification and network compression tasks, the proposed VBD shows superior performance in terms of regularizing network training.

VBD is a general dropout framework, which exploits a promising direction for dropout, \ie, investigating hierarchical priors of network weights. The Gaussian-uniform prior in this study is one feasible choice, but no means the only option. In the future, more effort will be made to exploit other possible choices to regularize the training process of deep neural networks better.

{{{\bf Acknowledgements:} This work was partially supported by National Key R\&D Program of China (2018YFB0904200), National Natural Science Foundation of China (No. 61672024, 61170305 and 60873114) and Australian Research Council (DP140102270 and DP160100703). Yuhang was supported by a scholarship from the China Scholarship Council.}}

{
\bibliographystyle{ieee}
\bibliography{mylibrary}
}

\end{document}